\documentclass{article}

\title{Second-order Quantile
Methods\linebreak for Experts and
Combinatorial Games}

\author{
  Wouter M. Koolen
  \and
  Tim {van Erven}
}

\usepackage{natbib}
\usepackage{amssymb}
\usepackage{amsthm}
\newtheorem{theorem}{Theorem}
\newtheorem{lemma}[theorem]{Lemma}
\usepackage{algcompatible}
\usepackage{algorithm}
\usepackage{fullpage}
\newcommand{\acks}[1]{}
\newenvironment{keywords}{\textbf{Keywords:}}{}

\DeclareRobustCommand{\VAN}[3]{#2} %

\usepackage{amsmath}
\usepackage{times} %
\usepackage{notation}
\usepackage{hyperref}
\usepackage{amsfonts}
\usepackage{bm}
\usepackage{xcolor}
\usepackage{accents} 
\usepackage{tikz}
\usetikzlibrary{shapes}
\usetikzlibrary{shadows}

\let\top\intercal
\let\loss\ell
\DeclareBoldMathCommand{\a}{a}
\DeclareBoldMathCommand{\b}{b}
\DeclareBoldMathCommand{\c}{c}
\DeclareBoldMathCommand{\w}{w}
\DeclareBoldMathCommand{\v}{v}
\DeclareBoldMathCommand{\vloss}{\loss}
\DeclareBoldMathCommand{\e}{e}
\DeclareBoldMathCommand{\r}{r}
\DeclareBoldMathCommand{\u}{u}
\DeclareBoldMathCommand{\x}{x}
\DeclareBoldMathCommand{\zeros}{0}
\DeclareBoldMathCommand{\ones}{1}
\DeclareBoldMathCommand{\vpi}{\pi}
\DeclareMathOperator*{\ex}{\mathbb E}

\DeclareMathOperator{\conv}{conv}

\let\min\indef
\DeclareMathOperator*{\min}{\vphantom{\sup}min}
\let\max\indef
\DeclareMathOperator*{\max}{\vphantom{\sup}max}
\DeclareMathOperator*{\argmin}{arg\,min}

\DeclareMathOperator{\erf}{erf}
\DeclareMathOperator{\re}{\triangle}

\renewcommand{\df}{=}

\newcommand{\U}{\mathcal U}
\newcommand{\reals}{\mathbb{R}}
\newcommand{\der}{\mathrm{d}}
\newcommand{\intder}{\,\der}

\newcommand{\refexperts}{\mathcal{K}}   %
\newcommand{\mleta}{\hat{\eta}}
\newcommand{\past}{\r_{1:T}}
\renewcommand{\half}{\frac{1}{2}}
\newcommand{\thalf}{\tfrac{1}{2}}

\begin{document}

\maketitle

\begin{abstract}
  We aim to design strategies for sequential decision making that adjust
  to the difficulty of the learning problem. We study this question both in
  the setting of prediction with expert advice, and for more general combinatorial decision tasks. We are not satisfied with just guaranteeing
  minimax regret rates, but we want our algorithms to perform
  significantly better on easy data. Two popular ways to formalize such
  adaptivity are second-order regret bounds and quantile
  bounds. The underlying notions of `easy data', which
  may be paraphrased as ``the learning problem has small
  variance'' and ``multiple decisions are useful'', are synergetic. But
  even though there are sophisticated algorithms that exploit one of the
  two, no existing algorithm is able to adapt to both.

In this paper we outline a new method for obtaining such adaptive
algorithms, based on a potential function that aggregates a range of
learning rates (which are essential tuning parameters). By choosing the
right prior
 we construct
efficient algorithms and show that they reap both benefits by proving
the first bounds that are both second-order and incorporate quantiles.
\end{abstract}


\section{Introduction}
We study the design of adaptive algorithms for online learning
\citep{CesaBianchiLugosi2006}. Our work starts in the \markdef{hedge
setting} \citep{FreundSchapire1997}, a core instance of prediction with
expert advice \citep{v-as-90,Vovk1998,LittlestoneWarmuth1994} and online
convex optimization \citep{Shalev-Shwartz2011}. Each round
$t=1,2,\ldots$ the learner plays a probability vector $\w_t$ on $K$
experts, the environment assigns a bounded loss to each expert in the
form of a vector $\vloss_t \in [0,1]^K$, and the learner incurs loss
given by the dot product $\w_t^\top \vloss_t$. The learner's goal is to
perform almost as well as the best expert, without making any
assumptions about the genesis of the losses. Specifically, the learner's
performance compared to expert $k$ is $r_t^k \df \w_t^\top \vloss_t -
\loss_t^k$, and after any number of rounds $T$ the goal is to have
small \markdef{regret} $R_T^k \df \sum_{t=1}^T r_t^k$ with respect to
every expert $k$.
\newcommand{\fancytag}[1]{%
\begin{tikzpicture}[overlay,remember picture,baseline=(X.base)]
 \node at (0,0 -| current page.south east) [xshift=-1.2in, anchor=east] (X) {#1};
\end{tikzpicture}
}

The \markdef{Hedge} algorithm by \citet{FreundSchapire1997} ensures %
\begin{equation}\label{eq:classic}
\fancytag{classic \phantom{\eqref{eq:classic}}}
R_T^k ~\prec~ \sqrt{T \ln K}
\qquad
\text{for each expert $k$}
\end{equation}
(with $\prec$ denoting moral inequality, i.e.\ suppressing details
inappropriate for this introduction), which is tight for adversarial
(worst-case) losses \citep{CesaBianchiLugosi2006}. Yet one can ask
whether the worst case is also the common case, and indeed
two lines of research show that this bound can be improved greatly
in various important scenarios. The first line of approaches
\citep{CesaBianchiMansourStoltz2007,hazan2010extracting,grad.variation,ftl.jmlr,GaillardStoltzVanErven2014}
obtains
\[
\tag{second order}
R_T^k ~\prec \sqrt{V_T^k \ln K}
\qquad
\text{for each expert $k$.}
\]
That is, $T$ can be reduced to $V_T^k$, which stands for some (there are
various) kind of cumulative variance or related second-order quantity.
This variance is then often shown to be small $V_T^k \ll T$ in important regimes like stochastic data (where it is typically bounded).
The second line, independently in parallel, shows how to reduce the
dependence on the number of experts $K$ whenever multiple experts
perform well. This is expected to occur, for example, when experts are  constructed by finely discretising the parameters of a (probabilistic) model, or when learning sub-algorithms are used as experts. The resulting so-called \emph{quantile bounds} (see \citealt{ChaudhuriFreundHsu2009,ChernovVovk2010,drifting}) of the form
\[
\tag{quantile}
\min_{k \in \refexperts} R_T^k ~\prec \sqrt{T \del[\big]{-\ln \pi(\refexperts)}}
\qquad
\text{for each subset $\refexperts$ of experts}
\]
improve $K$ to the reciprocal of the combined prior mass
$\pi(\refexperts)$, at the cost of now comparing to the worst expert among
$\refexperts$, so intuitively the best guarantee is obtained for $\refexperts$ the set of
``sufficiently good'' experts. (There is no requirement that the prior
$\pi(k)$ is uniform, and consequently quantile bounds imply the closely
related bounds with non-uniform priors, studied e.g.\ by
\citealt{HutterPoland2005}.) As these two types of improvements are
complementary, we would like to combine them in
a single algorithm. 
However, the mentioned two approaches are based on incompatible
techniques, which until now have refused to coexist.

\newcommand{\llrcost}{C_\text{lr}}
\paragraph{First Contribution}
We develop a new method, based on priors on a parameter called the learning rate and on experts, to
derive
algorithms and prove bounds that combine quantile and variance
guarantees. Our new prediction strategy, called \emph{Squint}, has regret at most
\begin{align}\label{eqn:intro.regret}
R_T^{\refexperts}  
&~\prec~  
\sqrt{V_T^{\refexperts} \del[\big]{\llrcost - \ln  \pi(\refexperts)}}
\qquad
\text{for each subset $\refexperts$ of experts}
,
\end{align}
where $R_T^\refexperts = \ex_{\pi(k|\refexperts)} R_T^k$ and
$V_T^\refexperts = \ex_{\pi(k|\refexperts)} V_T^k$ denote the average
(under the prior)
among the reference experts $k \in \refexperts$ of the regret
$R_T^k = \sum_{t=1}^T r_t^k$ and the (uncentered) variance of the excess
losses $V_T^k =
\sum_{t=1}^T (r_t^k)^2$. 
The overhead $\llrcost$ for
learning the optimal learning rate is specified below.

As is common for this type of results,
our variance factor $V_T^\refexperts$ is opaque in that it depends on the algorithm as well as the
data, yet our bound does imply small regret is several important cases. For example, we
immediately see that variance and hence regret stop accumulating
whenever the weights concentrate, as will happen when 
one expert is clearly better than all the others. (The expert loss variance of \cite{hazan2010extracting} does depend only on the data, but unfortunately may grow linearly even when the best expert is obvious.)
Furthermore, \citet{GaillardStoltzVanErven2014} show
that second-order bounds like \eqref{eqn:intro.regret} imply small
regret over experts with small losses ($L_T^*$-bounds) and bounded regret both in expectation and with high probability in stochastic settings with a unique best expert.

We will instantiate our scheme three times, varying the prior
distribution of the learning rate, to obtain three interesting bounds. 
First, for the \emph{uniform prior}, we obtain an efficient algorithm
with $\llrcost = \ln V_T^{\refexperts}$. Then we consider a prior that
we call the \emph{CV prior}, because it was introduced by
\citet{ChernovVovk2010} (to get
quantile bounds), and we improve the bound to
$\llrcost = \ln \ln V_T^{\refexperts}$. As we consider $\ln(\ln(x))$ to
be essentially constant, this algorithm achieves our goal of combining
the benefits of second-order bounds with quantiles, but unfortunately it
does not have an efficient implementation. Finally, by considering the
improper(!) \emph{log-uniform prior}, we get the best of both worlds: an
algorithm that is both efficient and achieves our goal with $\llrcost =
\ln \ln T$.
The efficient algorithms for the uniform and the log-uniform prior both
perform just $K$ operations per round, and are hence as widely
applicable as vanilla Hedge.

\paragraph{Combinatorial games}
We then consider a more sophisticated setting, where instead of experts $k
\in \set{1,\ldots,K}$, the elementary actions are combinatorial concepts
from some class $\mathcal C \subseteq \set{0,1}^K$. Many theoretically
interesting and important real-world online decision problems are of
this form, for example subsets (sub-problem of Principal Component
Analysis), lists (or ranking), permutations (scheduling), spanning trees
(communication), paths through a fixed graph (routing), etc.\ (see for
instance \citealt{paths,fpl,pca,perm,jcomband,jlists,online.comb}). The
combinatorial structure is reflected in the loss, which decomposes into
a sum of coordinate losses. That is, the loss of concept $\c \in \mathcal C$ is $\c^\top \vloss$ for some loss vector $\vloss \in [0,1]^K$.  This is natural: for example the loss of a path is the total loss of its edges.
\cite{koolen10:_hedgin_struc_concep} develop Component Hedge (of the Mirror Descent family), with regret at most
\begin{equation}\label{eq:CH}
R_T^\c 
~\prec~
\sqrt{T K \text{comp}(\mathcal C)}
\qquad
\text{for each concept $\c \in \mathcal C$}
,
\end{equation}
where $\text{comp}(\mathcal C)$, the analog of $\ln K$ for experts, measures the complexity (entropy) of the combinatorial class $\mathcal C$. \cite{drifting} derive $\sqrt{T}$ quantile bounds for online convex optimization. No combinatorial second-order quantile methods are known for combinatorial games. 

\paragraph{Second Contribution}
We extend our method to combinatorial games and obtain algorithms with second-order quantile regret bounds. Our new predictor \emph{Component iProd} keeps the regret below
\begin{equation}\label{eq:intro.CSquint}
R_T^\v
~\prec~
\sqrt{
  V_T^\v
  \del[\big]{
    \text{comp}{(\v)}
    + K \llrcost
  }
}
\qquad
\text{for each $\v \in \conv(\mathcal C)$}
.
\end{equation}
In the combinatorial domain the role of the reference set of experts $\refexperts$ is subsumed by an ``average concept'' vector $\v \in \conv(\mathcal C)$, for which our bound relates the coordinate-wise average regret $R_T^\v = \sum_{t,k} v_k r_t^k$ to the averaged variance $V_T^\v = \sum_{t,k} v_k (r_t^k)^2$ and the prior entropy $\text{comp}(\v)$.

Even if we disregard computational efficiency,
our bound \eqref{eq:intro.CSquint} is not a straightforward consequence
of the experts bound \eqref{eqn:intro.regret} applied with one expert
for each concept, paralleling the fact that \eqref{eq:CH} does not
follow from \eqref{eq:classic}. The reason is that we would obtain a
bound with per-concept variance $\sum_t (\sum_k v_k r_t^k)^2$ instead,
which can overshoot even the straight-up worst-case bound \eqref{eq:CH}
by a $\sqrt{K}$ factor (\cite{koolen10:_hedgin_struc_concep} call this
the \emph{range factor problem}). To avoid this problem, our method is
``collapsed'' (like Component Hedge): it only maintains first and second
order statistics about the $K$ coordinates
separately, not about concepts as a whole

\subsection{Related work}\label{sec:relatedwork}
Obtaining bounds for easy data in the experts setting is typically achieved by adaptively tuning a \emph{learning rate}, which is a parameter found in many algorithms. Schemes for \emph{choosing} the learning rate on-line are built by \cite{AuerCesaBianchiGentile2002, CesaBianchiMansourStoltz2007,hazan2010extracting,ftl.jmlr,GaillardStoltzVanErven2014,Wintenberger2014Arxiv}. These schemes typically choose a monotonically decreasing sequence of learning rates to prove a certain regret bound. 

Other approaches try to \emph{aggregate} multiple learning rates. The
motivations and techniques here show extreme diversity, ranging from
\emph{drifting games} by \cite{ChaudhuriFreundHsu2009, drifting}, and \emph{defensive forecasting} by \cite{ChernovVovk2010} to \emph{minimax relaxations} by \cite{relaxations} and \emph{budgeted timesharing} by \cite{learning.learning.rate}. The last scheme is of note, as it does not aggregate to reproduce a bound of a certain form, but rather to compete with the optimally tuned learning rate for the Hedge algorithm.

\paragraph{Outline}
We introduce the Squint prediction rule for experts in
Section~\ref{sec:Squint}. In Section~\ref{sec:priors} we motivate three
choices for the prior on the learning rate, discuss the resulting
algorithms and prove second-order quantile regret bounds. In
Section~\ref{sec:combinatorial} we extend Squint to combinatorial prediction tasks. We conclude with open problems in Section~\ref{sec:conclusion}.

\section{\emph{Squint}: a Second-order Quantile Method for Experts}\label{sec:Squint}

Let us review the expert setting protocol to fix notation. In round $t$ the algorithm plays a probability distribution $\w_t$ on
$K$ experts and encounters loss $\vloss_t \in [0,1]^K$.
The \emph{instantaneous
regret} of the algorithm compared to expert $k$ is
$r_t^k \df \w_t^\top \vloss_t - \loss_t^k = (\w_t - \e_k)^\top
\vloss_t$, where $\e_k$ is the unit vector in direction $k \in \{1,\ldots,K\}$.
Let
$R_T^k \df \sum_{t=1}^T r_t^k$ be the total regret compared to expert
$k$ and let $V_T^k \df \sum_{t=1}^T (r_t^k)^2$ be the cumulative uncentered
variance of the instantaneous regrets.

The central building block of our approach is a \emph{potential
function} $\Phi$ that maps sequences of instantaneous regret vectors
$\r_{1:T} = \tuple{\r_1, \ldots, \r_T}$ of any length $T \ge 0$ to
numbers. Potential functions are staple online learning tools
\citep{CesaBianchiLugosi2006,smoothing}.
We advance the following schema, which we call \markdef{Squint} (for \emph{second-order quantile integral}). It consists of the potential function and associated prediction rule
\begin{align}\label{eq:Squint}
\Phi(\r_{1:T})
&~\df~
\ex_{\pi(k)\gamma(\eta)} \sbr*{ e^{\eta R_T^k - \eta^2 V_T^k} - 1}
,
&
\w_{T+1}
&~\df~
\frac{
  \ex_{\pi(k)\gamma(\eta)} \sbr*{ e^{\eta R_T^k - \eta^2 V_T^k} \eta \e_k}
}{
  \ex_{\pi(k)\gamma(\eta)} \sbr*{ e^{\eta R_T^k - \eta^2 V_T^k} \eta}
  \phantom{\!\e_k}
}
,
\end{align}
where the expectation is taken under prior distributions $\pi$ on
experts $k \in \set{1, \ldots, K}$ and $\gamma$ on learning rates $\eta
\in [0,1/2]$ that are parameters of Squint. We will see in a moment that Squint ensures that the potential
remains $\Phi(\past) \le 0$ non-positive. Let us first investigate why
non-positivity is desirable. To gain a quick-and-dirty appreciation for this,
suppose that $\refexperts \subseteq \{1,\ldots,K\}$ is the
\emph{reference set} of experts with good performance.
Let us abbreviate their average regret and variance to $R \df R_T^\refexperts \df \ex_{\pi(k|\refexperts)} R_T^k$ and $V \df V_T^\refexperts \df \ex_{\pi(k|\refexperts)} V_T^k$.
Furthermore, imagine for simplicity that the prior $\gamma$ puts all its
mass on learning rate $\hat \eta = \frac{R}{2 V}$. Now non-positive potential $\Phi(\past) \le 0$ implies
\begin{equation}\label{eqn:oracleeta}
1
~\ge~
\ex_{\pi(k)} \sbr*{ e^{\hat\eta R_T^k - \hat\eta^2 V_T^k}}
~\ge~
\pi(\refexperts) \ex_{\pi(k|\refexperts)} \sbr*{ e^{\hat\eta R_T^k - \hat\eta^2 V_T^k}}
~\stackrel{\text{\tiny Jensen}}{\ge}~
\pi(\refexperts) e^{\hat\eta R - \hat\eta^2 V}
~=~
\pi(\refexperts) e^{\frac{R^2}{4 V}},
\end{equation}
which immediately yields the desired variance-with-quantiles bound
\[
R_T^\refexperts
~\le~
2 \sqrt{
  V_T^\refexperts (-\ln \pi(\refexperts))
}
.
\tag{Back of envelope, but great promise!}
\]
This raises the question: how does Squint keep
its potential $\Phi$ below zero? By always decreasing it:
\begin{lemma}\label{lem:supermartless1}
Squint \eqref{eq:Squint}
ensures that, for any loss sequence $\vloss_1, \ldots, \vloss_T$ in
$[0,1]^K$,
\begin{equation}\label{eqn:supermartless1}
\Phi(\r_{1:T})
~\le~
\ldots
~\le~
\Phi(\emptyset) ~=~ 0
.
\end{equation}
\end{lemma}
\begin{proof}
The key role is played by the upper bound \citep[Lemma~2.4]{CesaBianchiLugosi2006}
\begin{align}\label{eq:prodbound}
e^{x-x^2} -1 ~\le~ x \qquad \text{for $x \ge -1/2$}.
\end{align}
Applying this to $\eta r_{T+1}^k \ge -1/2$, we bound the increase  $\Phi(\r_{1:T+1}) - \Phi(\r_{1:T})$ of the potential by
\begin{align*}
\ex_{\pi(k)\gamma(\eta)} \sbr*{
  e^{\eta R_T^k - \eta^2 V_T^k} \del*{
    e^{\eta r^k_{T+1} -  (\eta r_{T+1}^k)^2}
    - 1
  }
}
\stackrel{\text{\tiny \eqref{eq:prodbound}}}{\le}
\ex_{\pi(k)\gamma(\eta)} \sbr*{ e^{\eta R_T^k - \eta^2 V_T^k} \eta (\w_{T+1}-\e_k)^\top \vloss_{T+1}}
=
0,
\end{align*}
where the last identity holds because the algorithm's weights \eqref{eq:Squint} have been
chosen to satisfy it.
\end{proof}

In Section~\ref{sec:priors} we will make the proof sketch from
\eqref{eqn:oracleeta} rigorous. The hunt is on for priors $\gamma$ on
$\eta$ that (a) pack plenty of mass close to $\hat \eta$, wherever it
may end up; and (b) admit efficient computation of the weights $\w_{T+1}$ by means of a closed-form formula for its integrals over $\eta$. We conclude this section by putting Squint in context.

\paragraph{Discussion}
The Squint potential is a function of the vector $\sum_{t=1}^T \r_t$ of
cumulative regrets, but also of its sum of squares, which is essential
for second-order bounds. Squint is an \emph{anytime} algorithm, i.e.\ it has no built-in dependence on an eventual time horizon, and its regret bounds hold at any time $T$ of evaluation. In addition Squint is \emph{timeless} in the sense of \cite{ftl.jmlr}, meaning that its predictions (current and future) are unaffected by inserting rounds with $\vloss = \zeros$. 

The Squint potential is an average of exponentiated negative ``losses'' (derived from the regret) under product prior $\pi(k)\gamma(\eta)$, reminiscent of the exponential weights analysis potential. Our Squint weights could be viewed as exponential weights, but, intriguingly, for \emph{another prior}, with $\gamma(\eta)$ replaced by $\gamma(\eta) \eta$. Mysteriously, playing the latter controls the former.

The bound \eqref{eq:prodbound} is hard-coded in our Squint potential
function and algorithm. To instead delay this bound to the analysis, we
might introduce the alternative \markdef{iProd} (for \emph{integrated products}) scheme
\begin{align}\label{eq:iprod}
\Phi(\r_{1:T})
&\df
\ex_{\pi(k)\gamma(\eta)} \sbr*{ \del*{ \prod_{t=1}^T (1+\eta r_t^k)} -1}
,
&
\w_{T+1}
&\df
\frac{
  \ex_{\pi(k)\gamma(\eta)} \sbr*{ \del*{\prod_{t=1}^T (1+\eta r_t^k)} \eta \e_k}
}{
  \ex_{\pi(k)\gamma(\eta)} \sbr*{ \del*{\prod_{t=1}^T (1+\eta r_t^k)} \eta
  }
  \phantom{\!\e_k}
}
.
\end{align}
The iProd weights keep the iProd potential identically zero, above the
Squint potential by \eqref{eq:prodbound}, and Squint's regret bounds
hence transfer to iProd. We champion Squint over the purer iProd because
Squint's weights admit efficient closed form evaluation, as shown in the next section.
For $\gamma$ a point-mass on a fixed choice of $\eta$ this advantage
disappears, and iProd reduces to Modified Prod by
\citet{GaillardStoltzVanErven2014}, whereas Squint becomes very similar to
the OBA algorithm of \citet{Wintenberger2014Arxiv}.

\section{Three Choices of the Prior on Learning Rates}\label{sec:priors}

We will now consider different choices for the prior $\gamma$ on $\eta
\in [0,1/2]$.
In each case the proof of the corresponding regret bound elaborates on the argument in
\eqref{eqn:oracleeta}, showing
that the priors place sufficient mass in a neighbourhood of the
optimized learning rate $\mleta$. This might be
viewed as performing a \emph{Laplace approximation}%
 of the integral over $\eta$, although the details
vary slightly depending on the prior $\gamma$. The prior $\pi$ on experts remains completely general. The proofs can be found in Appendix~\ref{appx:proofs}.

\subsection{Conjugate Prior}
First we consider a conjugate prior $\gamma$ with density 
\begin{equation}\label{eqn:conjugateprior}
\frac{\der \gamma}{\der \eta} = \frac{e^{a \eta - b \eta^2}}{Z(a,b)}
\qquad
\text{where}
\qquad 
Z(a,b) ~\df~ \int_0^{1/2} e^{a \eta- b \eta^2} \der \eta
\end{equation}
for parameters $a,b \in \reals$.  The uniform prior, mentioned in the
introduction, corresponds to the special case $a = b = 0$, for which
$Z(a,b)=1/2$. Abbreviating $x = a + R_T^k$ and $y = b + V_T^k$, the
Squint predictions \eqref{eq:Squint} then specialize to become
proportional to  
\begin{equation}\label{eq:conjugate.weights}
    w_{T+1}^k
      \propto \pi(k) \int_0^{\frac{1}{2}} e^{\eta x - \eta^2 y} \eta \intder \eta
      = \pi(k)\left(
        \frac{e^{\frac{x^2}{4y}} \sqrt{\pi} x
        \left(\erf\left(\frac{x}{2\sqrt{y}}\right) -
        \erf\left(\frac{x-y}{2\sqrt{y}}\right)\right)}{4 y^{3/2}}
        +\frac{1 - e^{\frac{x}{2} - \frac{y}{4}}}{2y}
         \right).
\end{equation}
These weights can be computed efficiently (but see
Appendix~\ref{eq:numerics} for numerically stable evaluation). 
For this prior, we obtain the following bound:
\begin{theorem}[Conjugate Prior]\label{thm:conjugate}
  Let $\ln_+(x) = \ln(\max\{x,1\})$.
  Then the regret of Squint \eqref{eq:Squint} with conjugate prior
  \eqref{eqn:conjugateprior} (with respect to any subset of experts
  $\refexperts$) is bounded by
  \begin{equation}\label{eqn:conjugateregret}
    R_T^\refexperts \le
         2\sqrt{\left(V_T^\refexperts + b\right)
        \left(\half + \ln_+
      \left(\frac{Z(a,b)\sqrt{2 (V_T^\refexperts +
      b)}}{\pi(\refexperts)}\right)\right)}
      + 5 \ln_+\left( \frac{2\sqrt{5}Z(a,b)}{\pi(\refexperts)}\right)
      - a.
  \end{equation}
\end{theorem}
The oracle tuning $a=0$ and $b = V_T^\refexperts$ results in $Z(a,b)
\le \frac{\sqrt{\pi}}{2 \sqrt{V_T^\refexperts}}$. Plugging this in we
find that the main term in \eqref{eqn:conjugateregret} becomes
\[
2\sqrt{2 V_T^\refexperts
  \left(\half + \ln_+
    \left(\frac{\sqrt{\pi}}{\pi(\refexperts)}\right)\right)},
\]
which is of the form \eqref{eqn:intro.regret} that we are after, with
constant overhead $\llrcost$ for learning the learning rate. Of course,
the fact that we do not know $V_T^\refexperts$ in advance makes this
tuning impossible, and for any constant parameters $a$ and $b$ we get a
factor of order $\llrcost = \ln V_T^\refexperts$.

\subsection{A Good Prior in Theory}
The reason the conjugate prior does not achieve the optimal bound is
that it does not put sufficient mass in a neighbourhood of the optimal
learning rate $\mleta$ that maximizes $e^{\eta R_T^\refexperts - \eta^2
V_T^\refexperts}$. To see how we could address this issue, observe that
we can plug $\alpha \mleta$ instead of $\mleta$ into
\eqref{eqn:oracleeta} for some scaling factor $\alpha \in (0,1)$, and
still obtain the desired regret bound up to a constant factor (which
depends on $\alpha$). This implies that, if we could find a prior that
puts a constant amount of mass on the interval $[\alpha \mleta,\mleta]$,
independent of $\mleta$, then we would only pay a constant cost $\llrcost$ to
learn the learning rate, at the price of having a slightly worse constant
factor.

A prior that puts constant mass on any interval $[\alpha \mleta,\mleta]$
should have a distribution function of the form $a \ln(\eta) + b$ for
some constants $a$ and $b$, and hence its density should be proportional
to $1/\eta$. But here we run into a problem, because, unfortunately,
$1/\eta$ does not have a finite integral over $\eta \in [0,1/2]$ and
hence no such prior exists!

The solution we adopt in this section will be to adjust the density
$1/\eta$ just a tiny amount so that it does integrate. Let $\gamma$ have
density
\begin{equation}\label{eqn:chernov.vovk.prior}
    \frac{\der \gamma}{\der \eta} = \frac{\ln 2}{\eta \ln^2(\eta)},
\end{equation}
where $\ln^2(x) = \big(\ln(x)\big)^2$. We call this the \emph{CV prior},
because it has previously been used to get quantile bounds by
\citet{ChernovVovk2010}. The additional factor $1/\ln^2(\eta)$ in the
prior only leads to an extra factor of $\sqrt{\ln \ln V_T^\refexperts}$
in the bound, which we consider to be essentially a constant.

Although the motivation above suggests that we might obtain a suboptimal
constant factor (depending on $\alpha$), a more careful analysis shows
that this does not even happen: apart from the effect of the
$1/\ln^2(\eta)$ term in prior, we obtain the optimal multiplicative
constant.

\begin{theorem}[CV Prior]\label{thm:logcauchy}
  Let $\ln_+(x) = \ln(\max\{x,1\})$. Then the regret of Squint \eqref{eq:Squint} with CV prior
  \eqref{eqn:chernov.vovk.prior} (with respect to any subset of experts
  $\refexperts$) is bounded by
  \begin{equation}\label{eqn:logcauchy}
    R_T^\refexperts \leq \sqrt{2V_T^\refexperts}\left(1 +
    \sqrt{2\ln_+\left(\frac{\ln_+^2\Big(\frac{2\sqrt{V_T^\refexperts}}{2-\sqrt{2}}\Big)}{\pi(\refexperts)\ln(2)}\right)}\right)
    - 5\ln \pi(\refexperts) + 4.
  \end{equation}
\end{theorem}

\subsection{Improper Prior}
In the last section we argued that we needed a density proportional to
$1/\eta$ on $\eta \in [0,1/2]$. Such a density would not integrate, and
we studied the CV prior density instead. However, we could be bold and see what breaks if we use the improper $1/\eta$ density anyway. We should be highly suspicious though, because this density is \emph{improper of the worst kind}: the integral $\int_0^{1/2} e^{\eta R_T- \eta^2 V_T} \frac{1}{\eta} \dif \eta$ diverges no matter how many rounds of data we process (a Bayesian would say: ``the posterior remains improper'').
Yet it turns out that we hit no essential impossibilities: the improper
prior $1/\eta$ cancels with the $\eta$ present in the Squint rule
\eqref{eq:Squint}, and the predictions are always well-defined.
As we will
see, we still get desirable regret bounds, but, equally important, we regain a closed-form integral for our weight computation. The Squint prediction \eqref{eq:Squint} specializes to
\begin{equation}\label{eq:Squint.improper}
w_{T+1}^k 
~\propto~
\pi(k) \int_0^{1/2}  e^{\eta R_T^k - \eta^2 V_T^k} \dif \eta
~=~
\pi(k)
\frac{
  \sqrt{\pi}
  e^{\frac{(R_T^k)^2}{4 V_T^k}} \left(\erf\left(\frac{R_T^k}{2 \sqrt{V_T^k}}\right)-\erf\left(\frac{R_T^k-V_T^k}{2 \sqrt{V_T^k}}\right)\right)
}{
  2 \sqrt{V_T^k}
}.
\end{equation}
(We look at numerical stability in Appendix~\ref{eq:numerics}.)
This strategy provides the following guarantee:
\begin{theorem}[Improper Prior]\label{thm:improper.bound}
The regret of Squint with improper prior \eqref{eq:Squint.improper}
(with respect to any subset of experts $\refexperts$) is bounded by
\begin{equation}\label{eqn:improper.bound}
R_T^\refexperts
  \leq \sqrt{2 V_T^\refexperts} \del*{1 + \sqrt{2\ln \del*{\frac{\half +
  \ln (T+1)}{\pi(\refexperts)}}}}
    + 5 \ln \del*{1 + \frac{1 +  2\ln (T+1)} {\pi(\refexperts)}}.
\end{equation}
\end{theorem}

\newcommand{\lcmix}{\ell_\text{mix}}

\section{\emph{Component iProd}: a Second-order Quantile Method for Combinatorial Games}\label{sec:combinatorial}
In the combinatorial setting the elementary actions are combinatorial
concepts from some class $\mathcal C \subseteq \set{0,1}^K$. The
combinatorial structure is reflected in the loss, which decomposes into
a sum of coordinate losses. That is, the loss of concept $\c \in
\mathcal C$ is $\c^\top \vloss$ for some loss vector $\vloss \in
[0,1]^K$. For example, the loss of a path is the total loss of its edges. We allow the learner to play a distribution $p$ on $\mathcal C$ and incur the expected loss
\(
\ex_{p(\c)} \sbr*{\c^\top \vloss}
=
\ex_{p(\c)} \sbr*{\c}^\top \vloss
\)
.
This means that the loss of $p$ is determined by its mean, which is
called the \markdef{usage} of $p$. We can therefore simplify the setup
by having the learner play a usage vector $\u \in \U$, where $\U \df
\conv(\mathcal C) \subseteq [0,1]^K$ is the polytope of valid usages.
The loss then becomes $\u^\top \vloss$. 

\citet{koolen10:_hedgin_struc_concep} point out that the Hedge algorithm with the concepts as experts guarantees
\begin{equation}\label{eq:EH}
\fancytag{Expanded Hedge \phantom{\eqref{eq:EH}}}
R_T^\c 
~\prec~
K \sqrt{T \text{comp}(\mathcal C)}
,
\end{equation}
upon proper tuning, where $\text{comp}(\mathcal C)$ is some appropriate
notion of the complexity of the combinatorial class $\mathcal C$. (This
is exactly \eqref{eq:classic}, where the additional factor $K$
comes from the fact that the loss of a single concept now ranges over
$[0,K]$ instead of $[0,1]$.) The computationally efficient Follow the Perturbed Leader strategy has the same bound.
However, \cite{koolen10:_hedgin_struc_concep} show
that this bound has a fundamentally suboptimal dependence on the loss
range, which they call the \emph{range factor problem}. Properly tuned,
their Component Hedge algorithm (a particular instance of Mirror
Descent) keeps the regret below
\begin{equation}\label{eq:CH2}
\fancytag{Component Hedge \phantom{\eqref{eq:CH2}}}
R_T^\c 
~\prec~
\sqrt{T K \text{comp}(\mathcal C)}
,
\end{equation}
the improvement being due to the algorithm exploiting the sum structure
of the loss. To show that this cannot be improved further,
\cite{koolen10:_hedgin_struc_concep} exhibit matching lower bounds for a
variety of combinatorial domains. \cite{online.comb} give an example
where the upper bound \eqref{eq:EH} for Expanded Hedge is tight, so the
range factor problem cannot be solved by a better analysis.

\bigskip
\noindent
In this section we aim to develop efficient algorithms for combinatorial
prediction that obtain the second-order and quantile improvements of
\eqref{eq:CH2}, but do not suffer from the range factor problem.

It is instructive to see that our bounds \eqref{eqn:intro.regret} for Squint/iProd, when applied with a separate
expert for each concept, indeed also suffer from a suboptimal loss range dependence. We find
\[
\fancytag{Expanded Squint/iProd}
R_T^\c
~\prec~
\sqrt{V_T^\c (\text{comp}(\mathcal C) + \text{tuning cost})},
\]
where $V_T^\c = \sum_{t=1}^T (r_T^\c)^2 = \sum_{t=1}^T (\sum_{k=1}^K
r_t^k)^2$ with $r_t^k \in [-1,1]$ may now be as large as $K^2 T$,
whereas we know $K T$ suffices. The reason for this is that 
$V_T^\c$ measures the variance of the concept as a whole, whereas the sum
structure of the loss makes it possible to replace $V_T^\c$ by the sum of
the variances of the components. In the analysis, this problem shows up
when we apply the bound \eqref{eq:prodbound}. To fix it, we must therefore rearrange the
algorithm to be able to apply \eqref{eq:prodbound} once per
component.

\paragraph{Outlook}
Our approach will be based on a new potential function that aggregates
over learning rates $\eta$ explicitly and over the concept class
$\mathcal C$ implicitly. Our inspiration for the latter comes from
rewriting the factor featuring inside the $\ex_{\gamma(\eta)}$ expectation in the iProd potential \eqref{eq:iprod} as
\begin{equation}\label{eqn:mixloss}
\ex_{\pi(k)}\sbr*{ \prod_{t=1}^T (1+\eta r_t^k)}
~=~
\prod_{t=1}^T \frac{
  \ex_{\pi(k)} \sbr*{ 
    \prod_{s=1}^t (1+\eta r_s^k)
  }
}{
  \ex_{\pi(k)}\sbr*{ 
    \prod_{s=1}^{t-1} (1+\eta r_s^k)
  }
}
~=~
e^{ - \sum_{t=1}^T \lcmix(p_t^\eta, \x_t^\eta)},
\end{equation}
which we interpret as the \markdef{mix loss} (see
\citealt{ftl.jmlr}) of the \markdef{exponential weights distribution}
$p_t^\eta$ on \markdef{auxiliary losses}:
\begin{align*}
\lcmix(p, \x)
&~\df~
-\ln
  \ex_{p(k)} \sbr*{ e^{- x^k}}
,
&
p_t^\eta(k)
&~\df~
\frac{
  \pi(k) 
  e^{- \sum_{s=1}^{t-1} x_s^{\eta,k}}
}{
  \ex_{\pi(k)}\sbr*{ 
  e^{- \sum_{s=1}^{t-1} x_s^{\eta,k}}
  }
}
,
&
x_t^{\eta,k}
&~\df~
- \ln (1+ \eta r_t^k)
.
\end{align*}
Thus, for each fixed $\eta$, we have identified a sub-module
 in which the
loss is the mix loss. It turns out that the Squint/iProd regret bounds can be
reinterpreted as arising from (quantile) mix loss regret bounds for
exponential weights in this
sub-module. For combinatorial games, we hence need to upgrade exponential
weights to a combinatorial algorithm for mix loss. No such algorithm was
readily available, so we derive a new algorithm that we call Component
Bayes (a variant of Component Hedge)
in Section~\ref{sec:CB}, and prove a quantile mix loss regret bound for
it. Then in Section~\ref{sec:CSquint} we show that Component iProd,
obtained by substituting Component Bayes for exponential weights in the
sub-module above, inherits all of iProd's desirable features. That is, by
aggregating the above sub-module over learning
rates the Component iProd predictor delivers low
second-order quantile regret. Component iProd is summarized as Algorithm~\ref{alg:CSquint}.
Proofs can be found in Appendix~\ref{appx:proofs}.

\newcommand{\grid}{\mathcal{G}}

\begin{algorithm}[t]
\caption{Component iProd. Required subroutines are the relative entropy
projection step (3) and the decomposition step (5). For polytopes $\U$
that can be represented by few linear inequalities these can be deferred
to general-purpose convex and linear optimizers. See
\cite{koolen10:_hedgin_struc_concep} for more details, and for ideas
regarding more efficient implementations for example concept classes.
}\label{alg:CSquint}
\renewcommand{\algorithmicrequire}{\textbf{Input:}}
\renewcommand{\algorithmiccomment}[1]{\hfill $\triangleright$~\textit{#1}}
\begin{algorithmic}[1]
\REQUIRE Combinatorial class $\mathcal C \subseteq \set{0,1}^K$ with convex hull $\U = \conv(\mathcal C)$ 
\REQUIRE Prior distribution $\gamma$ on a discrete grid $\grid \subset
[0,1/2]$ and prior vector $\vpi \in [0,1]^K$
\STATE For each $\eta \in \grid$, initialize $\tilde \u_1^\eta = \vpi$ and $L_1^\eta = - \ln \del[\big]{\gamma(\eta)\eta}$
\COMMENT{\eqref{eq:CB.weigths}}
\FOR{$t=1,2,\ldots$}
  \STATE For each $\eta \in \grid$, project $\u_t^\eta = \min_{\u \in \U} \re_2\delcc{\u}{\tilde \u_t}$
\COMMENT{\eqref{eq:CB.project}}
  \STATE Compute usage $\u_{t}
=
\wfrac{
  \sum_\eta e^{- L^\eta_t} \u_t^{\eta}
}{
  \sum_\eta e^{- L^\eta_t}
}$
\COMMENT{\eqref{eq:CSquint}}
\STATE Decompose $\u_t = \sum_i p_i \c_{t,i}$ into a convex combination of concepts $\c_{t,i} \in \mathcal C$
\STATE Play $\c_{t,i}$ with probability $p_i$
\STATE Receive loss vector $\vloss_t$, incur expected loss $\u_t^\top \vloss_t$
\STATE For each $\eta \in \grid$ and $k$, update $\tilde u_{t+1}^{\eta,k} = 
u_t^{\eta,k}
\frac{
   1 + \eta \del{u_t^k - 1} \loss_t^k
}{
  1 + \eta \del{u_t^k - u_t^{\eta,k}} \loss_t^k
}
$
\COMMENT{\eqref{eq:CB.unconstrained},\eqref{eq:unc.min},\eqref{eq:loss.taste},\eqref{eq:def.regret}}

\STATE For each $\eta \in \grid$, update $L_{t+1}^\eta = L_t^\eta - \sum_{k=1}^K \frac{1}{K} \ln \del[\big]{1 + \eta \del[\big]{u_t^k - u_t^{\eta,k}} \loss_t^k}$
\COMMENT{\eqref{eq:mix.loss.def},\eqref{eq:loss.taste},\eqref{eq:def.regret}}

\ENDFOR
\end{algorithmic}
\end{algorithm}

\subsection{Component Bayes}\label{sec:CB}

In this section we describe a combinatorial algorithm for
mix loss, which will then be an essential
subroutine in our Component iProd algorithm. We take as our action
space some closed convex $\U \subseteq [0,1]^K$. The game then proceeds in rounds. Each round $t$ the learner plays $\u_t \in \U$, which we interpret as making $K$ independent plays in $K$ parallel two-expert sub-games, putting weight $u_t^k$ and $1-u_t^k$ on experts $1$ and $0$ in sub-game $k$. The environment reveals a loss vector $\x_t \in \reals^{K \times \set{0,1}}$ (we use $\x$ for the loss in this auxiliary game, and reserve $\vloss$ for the loss in the main game), and the loss of the learner is the sum of per-coordinate mix losses:
\begin{equation}\label{eq:mix.loss.def}
\lcmix(\u, \x)
~\df~
\sum_{k=1}^K - \ln \del*{u^k e^{-x^{k,1}} + (1-u^k) e^{- x^{k,0}}}
.
\end{equation}
The goal is to compete with the best element $\v \in \U$.
We define \markdef{Component Bayes}\footnote{
Ignoring a small technically convenient switch from generalized to binary relative entropy we find that Component Bayes \emph{equals} Component Hedge of \cite{koolen10:_hedgin_struc_concep}. The new name stresses an important distinction in the game protocol:  Component Hedge guarantees low \emph{linear loss} regret, Component Bayes guarantees low \emph{mix loss} regret.
} inductively as follows. We set $\tilde \u_1 = \vpi \in [0,1]^K$ to
some prior vector of our choice (which does not have to be a usage in
$\U$), and then alternate
\begin{subequations}
\label{eq:CB.weigths}
\begin{align}
\label{eq:CB.project}
\u_t &~\df~ \argmin_{\u \in \U} \re_2 \delcc{\u}{\tilde \u_t}
\\
\label{eq:CB.unconstrained}
\tilde \u_{t+1} &~\df~ \argmin_{\u \in [0,1]^K}
\re_2\delcc{\u}{\u_t} + \sum_{k=1}^K (u^k x_t^{k,1} + (1-u^k)
x_t^{k,0}),
\end{align}
\end{subequations}
where $\re_2$ denotes the binary relative entropy,
 defined from scalars $x$ to $y$ and vectors $\v$ to $\u$ by
\begin{align*}
\re_2 \delcc{x}{y} &~\df~ x \ln \frac{x}{y} + (1-x) \ln \frac{1-x}{1-y}
&
\text{and}
&
&
\re_2 \delcc{\v}{\u} &~\df~ \sum_{k=1}^K \re_2\delcc{v^k}{u^k}
.
\end{align*}
This simple scheme is all it takes to adapt to the combinatorial domain.

\begin{lemma}\label{lem:CB.bound}
Fix any closed convex $\U \subseteq [0,1]^K$. For any loss sequence $\x_1, \ldots, \x_T$ in $\reals^{K \times \set{0,1}}$, the mix loss regret of Component Bayes \eqref{eq:CB.weigths} with prior $\vpi \in [0,1]^K$ compared to any $\v \in \U$ is at most
\[
\sum_{t=1}^T \lcmix(\u_t, \x_t)
- \sum_{t=1}^T \sum_{k=1}^K \del*{v^k x_t^{k,1} + (1-v^k) x_t^{k,0}}
~\le~
\re_2\delcc{\v}{\vpi}.
\]
\end{lemma}
The practicality of Component Bayes does depend on the computational cost of computing the binary relative entropy projection onto the convex set $\U$. Fortunately, in many applications
 $\U$ has a compact representation by means of a few linear inequalities; e.g.\  the Birkhoff polytope
(permutations) and the Flow polytope (paths). See
\citep{koolen10:_hedgin_struc_concep} for examples. Component Bayes may then be implemented using off-the-shelf convex optimization subroutines like CVX.

\subsection{Component iProd}\label{sec:CSquint}
We now return to our original problem of combinatorial prediction with linear loss. Using Component Bayes (which is for mix loss) as a sub-module, we construct an algorithm with second-order quantile bounds. We first have to extend our notion of regret vector\footnote{
Interestingly, the natural generalization of the expert regret vector $r_t^k = (\u_t/K-\e_k)^\top
\vloss_t$, which renormalizes the usage, does \emph{not} result in the
desired result. To see this, consider a perfect scenario with a clearly best concept $\c \in \mathcal C$ on which the learner fully concentrates its predictions $\u_t = \c$. This should not result in any regret compared to $\c$. But for $k \in \c$ we still have $r^k_t \neq 0$ (unless all coordinates $k \in \c$ suffer identical loss), and so the variance may accumulate linearly. 
}. Suppose the learner predicts usage $\u_t \in \U \subseteq [0,1]^K$
and encounters loss vector $\vloss_t \in [-1,+1]^K$. We then define the regret vector $\r_t \in \mathcal [-1,+1]^{K \times \set{0,1}}$ by
\begin{align}
\label{eq:def.regret}
r_t^{k,1} &~\df~ u_t^k \loss_t^k - \loss_t^k
&
\text{and}
&&
r_t^{k,0} &~\df~ u_t^k \loss_t^k.
\end{align}
Fix a prior vector $\vpi \in [0,1]^K$ and prior distribution $\gamma$ on
$[0,1/2]$. We define the \markdef{Component iProd} potential function and predictor by
\begin{align}\label{eq:CSquint}
\Phi(\r_{1:T})
&~\df~
\ex_{\gamma(\eta)} \sbr*{
  e^{- \frac{1}{K} \sum_{t=1}^T \lcmix(\u_t^\eta, \x_t^\eta)}
  - 1
}
,
&
\u_{T}
~\df~
\frac{
  \ex_{\gamma(\eta)} \sbr*{
    e^{- \frac{1}{K} \sum_{t=1}^{T-1} \lcmix(\u_t^\eta, \x_t^\eta)}
    \eta \u_T^{\eta}
  }
}{
  \ex_{\gamma(\eta)} \sbr*{
    e^{- \frac{1}{K} \sum_{t=1}^{T-1} \lcmix(\u_t^\eta, \x_t^\eta)}
    \eta
  }
  \phantom{\!\u_T^{\eta}}
}
,
\end{align}
where $\u_1^\eta, \u_2^\eta, \ldots$ denote the usages of Component Bayes with prior $\vpi$ on losses $\x_1^\eta, \x_2^\eta,\ldots$ set to\footnote{
We could alternatively set $x_t^{\eta,k,b} \df \eta^2 (r_t^{k,b})^2 - \eta r_t^{k,b}$ and prove the same regret bound. But to get the tighter algorithm we delay the bound \eqref{eq:prodbound} to the analysis. See the discussion surrounding \eqref{eq:iprod} about Squint vs iProd.
}
\begin{equation}\label{eq:loss.taste}
x_t^{\eta,k,b}
~\df~
- \ln \del*{1+ \eta r_t^{k,b}}
\end{equation}
Note that $\u_T \in \U$ is a bona fide action, as it is a convex combination of $\u_T^\eta \in \U$.
As can be seen from
\eqref{eqn:mixloss}, this potential generalizes
the iProd \eqref{eq:iprod} potential: in the base case $K=1$ and $\mathcal C= \set{0,1}$
Component iProd reduces to iProd \eqref{eq:iprod} on $K=2$ experts if we set the loss for Component iProd to the difference of the losses for iProd.
We will now show that Component iProd has the desired regret guarantee.

\begin{lemma}\label{lem:CSquint.negpot}
Fix any closed convex $\U \subseteq [0,1]^K$. 
Component iProd \eqref{eq:CSquint} ensures that for any loss sequence $\vloss_1, \ldots, \vloss_T$ in $[-1,+1]^K$ we have
\(
\Phi(\past) \le \ldots \le  \Phi(\emptyset) 
=~
0
\)
.
\end{lemma}
We now establish that non-positive potential implies our desired regret
bound. We express our quantile bound in terms of the $\v$-weighted
cumulative coordinate-wise regret and uncentered variance %
\begin{alignat*}{2}
R_T^\v
&~\df~
\sum_{t=1}^T \sum_{k=1}^K
\del[\big]{
  v^k r_t^{k,1}
  + (1-v^k) r_t^{k,0}
}
&&
~\stackrel{\text{\tiny \eqref{eq:def.regret}}}{=}~
\sum_{t=1}^T (\u_t - \v)^\top \vloss_t
,
\\
V_T^\v &~\df~
\sum_{t=1}^T \sum_{k=1}^K
\del[\big]{
  v^k (r_t^{k,1})^2
  +  (1-v^k) (r_t^{k,0})^2
}
.
\end{alignat*}

\begin{lemma}\label{lem:CSquint.regret.bound}
Suppose $\gamma$ is supported on a discrete grid $\grid \subset [0,1/2]$. 
Component iProd \eqref{eq:CSquint} guarantees that for every $\eta \in \grid$ and for every comparator $\v \in \U$ the regret is at most
\[
\eta R_T^\v
- \eta^2 V_T^\v
~\le~
\re_2\delcc{\v}{\vpi}
- K \ln \gamma(\eta)
.
\]
\end{lemma}
We now discuss the choice of the discrete prior $\gamma$ on $\eta$. Here
we face a trade-off between regret and computation. More discretization
points reduce the regret overhead for mis-tuning, but since we need to
run one instance of Component Bayes per grid point the computation time
also grows linearly in the number of grid points. Fortunately,
Lemma~\ref{lem:CSquint.regret.bound} implies that exponential spacing
suffices, as missing the optimal tuning $\hat \eta = \frac{R_T^\v}{2
V_T^\v}$ by a constant factor affects the regret bound by another
constant factor. To see this, apply Lemma~\ref{lem:CSquint.regret.bound}
to $\eta = \alpha \hat \eta$. We find
\begin{equation}\label{eqn:CSquint.discretised.regret.bound}
R_T^\v
~\le~
\frac{2}{\sqrt{\alpha(2-\alpha)}}
\sqrt{
  V_T^\v
  \del[\big]{
    \re_2\delcc{\v}{\vpi}
    - K \ln \gamma(\alpha \hat \eta)
  }
}
.
\end{equation}
It is therefore sufficient to choose $\eta$ from an exponentially spaced
grid $\grid$. In particular, we propose to let $\gamma$ be the uniform
distribution on
\begin{equation}\label{eqn:grid}
  \grid = \{2^{-i} \mid i=1, \ldots, \ceil{1 + \log_2 T}\}.
\end{equation}
This leads to the following final regret bound:
\begin{theorem}\label{thm:CSquint.final.regret.bound}
  Let $\U \subseteq [0,1]^K$ be closed and convex. Component iProd
  \eqref{eq:CSquint}, with $\gamma$ the uniform prior on grid $\grid$
  from \eqref{eqn:grid} and arbitrary $\vpi \in [0,1]^K$, ensures that,
  for any sequence $\vloss_1 \ldots, \vloss_T$ of $[-1,+1]$-valued loss
  vectors, the regret compared to any $\v \in \U$ is at most
  \begin{equation}\label{eqn:CSquint.final.regret.bound}
    R_T^\v
    \le
    \frac{4}{\sqrt{3}}
    \sqrt{
      V_T^\v
      \del[\big]{
        \re_2\delcc{\v}{\vpi}
        + K \ln \ceil{1 + \log_2 T}
      }
    }
    + 4\re_2\delcc{\v}{\vpi}
    + K \max\{4\ln \ceil{1 + \log_2 T},1\}.
  \end{equation}
\end{theorem}

\paragraph{Discussion of Component iProd}\label{sec:CSquint.discussion}

We showed that if we have an algorithm for keeping the \emph{mix loss}
regret small compared to some concept class, we can run multiple
instances, each with a different learning rate factored into the losses,
and as a result also keep the linear loss small with second order
quantile bounds. Another setting where this could be applied is to
switching experts. The Fixed Share algorithm by
\cite{HerbsterWarmuth1998} applies to all Vovk mixable losses, so in
particular to the mix loss, and delivers adaptive regret bounds
\citep{adaptive.regret}. Aggregating over $\log_2 T$ exponentially
spaced $\eta$ to learn the \emph{learning rate} would indeed be very
cheap. 
Yet another setting is matrix-valued prediction under linear loss \citep{meg}, where our method would transport the mix loss bounds of \cite{matrixbayes} to second-order quantile bounds.

In Lemma~\ref{lem:CSquint.regret.bound} we see
that the cost $- \ln \gamma(\eta)$ for learning the learning rate $\eta$
occurs multiplied by the ambient dimension $K$. Intuitively
this seems wasteful, as we are not trying to learn a separate rate for
each component. But we could not reduce $K$ to $1$. For example, 
defining the potential \eqref{eq:CSquint} \emph{without}
the division by $K$ escalates its dependency on the
loss $\vloss$ from linear to polynomial of order $K$. Unfortunately this
potential cannot be kept below zero even for $K=2$.

\section{Conclusion and Future Work}\label{sec:conclusion}

We have constructed second-order quantile methods for both the expert
setting (Squint) and for general combinatorial games (Component iProd). The key
in both cases is the ability to learn the appropriate learning rate,
which is reflected by the integrals over $\eta$ in our potential
functions \eqref{eq:Squint} and \eqref{eq:CSquint}. As discussed in
Section~\ref{sec:relatedwork}, there is a whole variety of different
ways to adapt to the optimal $\eta$. This raises the question of whether
there is a unifying perspective that explains when and how it is possible to
learn the learning rate in general.

Another issue for future work is to find matching lower bounds.
Although lower bounds in terms of $\sqrt{T \ln K}$ are available for the worst
possible sequence \citep{CesaBianchiLugosi2006}, the issue is substantially more complex when considering either variances or quantiles. We are not aware of any
lower bounds in terms of the variance $V_T^k$. \citet{GoferMansour2012}
provide lower bounds  that hold for \emph{any}
sequence, in
terms of the squared loss ranges in each round, but these do not apply to methods that
adaptively tune their learning rate.
For quantile bounds, \cite{pareto_regret} takes a first step by characterizing the Pareto optimal quantile bounds for $2$ experts in the $\sqrt{T}$ regime.

Finally, we have assumed throughout that all losses are normalized to the range
$[0,1]$. But there exist second-order methods that do not require this
normalization and can adapt automatically to the loss range
\citep{CesaBianchiMansourStoltz2007,ftl.jmlr,Wintenberger2014Arxiv}. It
is an open question how such adaptive techniques can be incorporated
elegantly into our methods.

\acks{We thank a bunch of people.}

\DeclareRobustCommand{\VAN}[3]{#3} %

\bibliography{colt2015}

\newpage
\appendix

\section{Proofs}\label{appx:proofs}
This section collects the proofs omitted from Sections~\ref{sec:priors} and~\ref{sec:combinatorial}.

\subsection{Theorem~\ref{thm:conjugate}}
\begin{proof}
  Abbreviate $R = R_T^\refexperts + a$ and $V = V_T^\refexperts + b$.
  Then from \eqref{eqn:supermartless1} and Jensen's inequality we obtain
  \begin{equation*}
    1 \geq \ex_{\pi(k)\gamma(\eta)} \sbr*{e^{\eta R_T^k - \eta^2 V_T^k}}
      \geq \frac{ \pi(\refexperts) \ex_{\pi(k|\refexperts)} \int_0^{1/2} 
      e^{\eta (R_T^k+a) - \eta^2 (V_T^k+b)} \intder \eta}{Z(a,b)}
      \geq \frac{\pi(\refexperts) \int_0^{1/2} 
      e^{\eta R - \eta^2 V} \intder \eta}{Z(a,b)}.
  \end{equation*}
  The $\eta$ that maximizes $\eta R - \eta^2 V$ is $\mleta =
  \frac{R}{2V}$. Without loss of generality, we can assume that
  $\mleta \geq \frac{1}{\sqrt{2V}} \geq 0$, because otherwise $R \leq
  \sqrt{2V}$, from which \eqref{eqn:conjugateregret} follows directly.
  Now let $[u,v] \subseteq [0,\frac{1}{2}]$ be any interval such that $v
  \leq \mleta$. Then, because $\eta R - \eta^2 V$ is non-decreasing
  in $\eta$ for $\eta \leq \mleta$, we have
  \begin{equation*}
    \int_0^{1/2} 
      e^{\eta R - \eta^2 V} \intder \eta
      \geq \int_u^v
      e^{\eta R - \eta^2 V} \intder \eta\\
      \geq 
      (v-u) e^{u R - u^2 V},
  \end{equation*}
  so that the above two equations imply
  \begin{equation}\label{eqn:conjugateinterval}
    u R - u^2 V \leq \ln
    \left(\frac{Z(a,b)}{\pi(\refexperts)(v-u)}\right).
  \end{equation}
  Suppose first that $\mleta \leq 1/2$. Then we take $v = \mleta$ and $u
  = \mleta - \frac{1}{\sqrt{2V}}$. Plugging these into
  \eqref{eqn:conjugateinterval} we obtain
  \begin{equation*}
    R \leq 2\sqrt{V \left(\half + \ln
      \left(\frac{Z(a,b)\sqrt{2 V}}{\pi(\refexperts)}\right)\right)},
  \end{equation*}
  which implies \eqref{eqn:conjugateregret}.
  Alternatively, we may have $\mleta > 1/2$, which is equivalent to $R
  > V$. Then the left-hand side of \eqref{eqn:conjugateinterval} is at
  least $u (1-u) R$ and hence we obtain
  \begin{align*}
    R &\leq \frac{1}{(1-u)u} \ln
      \left(\frac{Z(a,b)}{\pi(\refexperts)(v-u)}\right).
  \end{align*}
  Taking $u = \frac{5 - \sqrt{5}}{10}$ and $v=1/2$ then leads to the
  bound
  \[
    R \leq 5 \ln
      \left(\frac{2\sqrt{5}Z(a,b)}{\pi(\refexperts)}\right),
  \]
  which again implies \eqref{eqn:conjugateregret}.
\end{proof}

\subsection{Theorem~\ref{thm:logcauchy}}
\begin{proof}
  Abbreviate $R = R_T^\refexperts$ and $V =  V_T^\refexperts$, and let
  $\mleta = \frac{R}{2V}$ be the $\eta$ that maximizes $\eta R - \eta^2
  V$. Then $\eta R - \eta^2 V$ is non-decreasing in $\eta$ for $\eta
  \leq \mleta$ and hence, for any interval $[u,v] \subseteq [0,1/2]$
  such that $v \leq \mleta$, we obtain from \eqref{eqn:supermartless1}
  and Jensen's inequality that
  \begin{equation}\label{eqn:logcauchygeneral}
    1 %
      \geq \pi(\refexperts) \ex_{\pi(k|\refexperts)\gamma(\eta)} \sbr*{e^{\eta R_T^k -
      \eta^2 V_T^k}}
      \geq \pi(\refexperts) \ex_{\gamma(\eta)} \sbr*{e^{\eta R - \eta^2 V}}
      \geq \pi(\refexperts) \gamma([u,v])
      e^{u R - u^2 V},
  \end{equation}
  where
  \begin{equation}\label{eqn:logcauchylower}
    \gamma([u,v]) = \int_u^v \frac{\ln 2}{ \eta (\ln \eta)^2} \intder \eta
        = \frac{\ln(2)}{\ln(\frac{1}{v})} -
        \frac{\ln(2)}{\ln(\frac{1}{u})}
        \geq \frac{\ln(2)\ln(\frac{v}{u})}{\ln^2(\frac{1}{u})}.
  \end{equation}
  If $R \leq 2\sqrt{V}$, then \eqref{eqn:logcauchy} follows by
  considering the cases $V \leq 4$ and $V \geq 4$, so suppose that $R
  \geq 2 \sqrt{V}$, which implies that $\mleta \geq
  \frac{1}{\sqrt{2 V}}$.

  Now suppose first that $\mleta \leq 1/2$. Then we take $v = \mleta$ and $u
  = \mleta - \frac{1}{\sqrt{2 V}} \geq 0$, for which
  \begin{equation*}
    e^{u R - u^2 V}\frac{\ln(\frac{v}{u})}{\ln^2(\frac{1}{u})}
      = \frac{e^{\frac{R^2}{4V} -
      \half}\ln\Big(\frac{1}{1-\frac{\sqrt{2V}}{R}}\Big)}
      {\ln^2(\frac{2V}{R-\sqrt{2V}})}
      \geq \frac{e^{\frac{R^2}{4V} -
      \half}\ln\Big(\frac{1}{1-\frac{\sqrt{2V}}{R}}\Big)}
      {\ln^2(\frac{2\sqrt{V}}{2-\sqrt{2}})},
  \end{equation*}
  where the inequality follows from $R \geq 2\sqrt{V}$. By $e^{\half(x^2
  -1)} = e^{\half(x-1)^2}e^{x-1} \geq e^{\half(x-1)^2}x$ and
  $\ln\frac{1}{1-x} \geq x$, we can lower bound the numerator with
  \begin{equation*}
    e^{\frac{R^2}{4V} - \half}\ln\Big(\frac{1}{1-\frac{\sqrt{2V}}{R}}\Big)
      \geq e^{\half (\frac{R}{\sqrt{2V}}-1)^2}
      \tfrac{R}{\sqrt{2V}}\tfrac{\sqrt{2V}}{R}
      = e^{\half (\frac{R}{\sqrt{2V}}-1)^2}.
  \end{equation*}
  Putting everything together, we obtain
  \begin{equation*}
    1 \geq \frac{\pi(\refexperts) \ln(2) e^{\half (\frac{R}{\sqrt{2V}}-1)^2}}
      {\ln^2\big(\frac{2\sqrt{V}}{2-\sqrt{2}} \big)},
  \end{equation*}
  which implies
  \begin{equation*}
    R \leq \sqrt{2V}\left(1 +
    \sqrt{2\ln\left(\frac{\ln^2\big(\frac{2\sqrt{V}}{2-\sqrt{2}}
    \big)}{\pi(\refexperts) \ln(2)}\right)}\right),
  \end{equation*}
  and \eqref{eqn:logcauchy} is satisfied.

  It remains to consider the case $\mleta > 1/2$, which implies $R > V$.
  Then we take $v=1/2$, and \eqref{eqn:logcauchygeneral} leads to 
  \begin{equation*}
    u R - u^2 V  \leq -\ln \pi(\refexperts) -\ln\Big(
      1 - \frac{\ln(2)}{\ln(\frac{1}{u})}\Big).
  \end{equation*}
  Using $R > V$, the left-hand side is at most $u(1-u)R$. The choice $u
  = \frac{5-\sqrt{5}}{10}$ then again implies \eqref{eqn:logcauchy},
  which completes the proof.
\end{proof}

\subsection{Theorem~\ref{thm:improper.bound}}

\begin{proof}
The proof of Lemma~\ref{lem:supermartless1} goes through unchanged for
the improper prior, but we have to be careful, because we cannot pull
out the constant $1$ from the integral over $\eta$ in the potential
function any more. So abbreviate $R = R_T^\refexperts$ and $V =
V_T^\refexperts$. Then, by \eqref{eqn:supermartless1}, $R \geq -T$, $V
\leq T$, and Jensen's inequality, \begin{align*}
0
~\ge~ 
\Phi(\past) 
&~=~
\ex_{\pi(k)}\sbr*{\int_0^{1/2} \frac{e^{\eta R_T^k - \eta^2 V_T^k}
-1}{\eta} \dif \eta}
\\
&~\ge~
\pi(\refexperts) \ex_{\pi(k\mid \refexperts)}\sbr*{\int_0^{1/2} \frac{e^{\eta
R_T^k - \eta^2 V_T^k} -1}{\eta} \dif \eta}
+
(1-\pi(\refexperts))  \int_0^{1/2} \frac{e^{-\eta T - \eta^2 T}
-1}{\eta} \dif \eta\\
&~\ge~
\pi(\refexperts) \int_0^{1/2} \frac{e^{\eta
R - \eta^2 V} -1}{\eta} \dif \eta
+
(1-\pi(\refexperts))  \int_0^{1/2} \frac{e^{-\eta T - \eta^2 T}
-1}{\eta} \dif \eta.
\end{align*}
Now first for the bad experts that are not in $\refexperts$, we will show
that
\begin{equation}\label{eqn:badexpertbound}
\int_0^{1/2} \frac{e^{-\eta T - \eta^2 T} -1}{\eta} \dif \eta ~\ge~ 
-\frac{1}{2} - \ln (T+1).
\end{equation}
Let $\epsilon \in [0,1/2]$ be arbitrary. Then, using $e^x \geq 1+x$ and $e^x
\geq 0$, we obtain 
\begin{align*}
  \int_0^{1/2} \frac{e^{-\eta T - \eta^2 T} -1}{\eta} \dif \eta
  &= \int_0^\epsilon \frac{e^{-\eta T - \eta^2 T} -1}{\eta} \dif \eta
  + \int_\epsilon^{1/2} \frac{e^{-\eta T - \eta^2 T} -1}{\eta} \dif \eta\\
  &\geq \int_0^\epsilon \frac{-\eta T - \eta^2 T}{\eta} \dif \eta
  +\int_\epsilon^{1/2} \frac{-1}{\eta} \dif \eta
  = -\epsilon T - \frac{\epsilon^2}{2} T
  +\ln(2\epsilon).
\end{align*}
The choice $\epsilon = \frac{1}{2(T+1)}$ implies
\eqref{eqn:badexpertbound} for all $T \geq 0$.

Second, for the good experts that are in $\refexperts$, we proceed as follows.
Let $\mleta = \frac{R}{2V}$ be the $\eta$ that maximizes $\eta R-\eta^2
V$. Then $\eta R - \eta^2 V$ is non-decreasing in $\eta$ for $\eta \leq
\mleta$ and hence, for any interval $[u,v] \subseteq [0,1/2]$ such that
$v \leq \mleta$,
\begin{align}
\int_0^{1/2} \frac{e^{\eta R - \eta^2 V} -1}{\eta} \dif \eta
&~\ge~
\int_0^u \frac{e^{0 R - 0 V} -1}{\eta}\dif \eta
+
(e^{u R - u^2 V}-1) \int_u^{v} \frac{1}{\eta} \dif \eta
-
\int_v^{1/2} \frac{1}{\eta} \dif \eta\notag
\\
&~=~
\del*{e^{u R - u^2 V}-1} \ln \frac{v}{u}
+
\ln (2 v)
.\label{eqn:NoNastyEta}
\end{align}
We may assume without loss of generality that $R \geq 2 \sqrt{V}$
(otherwise \eqref{eqn:improper.bound} follows directly), which implies
that $\mleta \geq \frac{1}{\sqrt{2V}}$.

We now have two cases. Suppose first that $\mleta \le 1/2$. Then we plug
in $v = \mleta$ and $u = \mleta - \frac{1}{\sqrt{2 V}}$ and use $R \geq
2 \sqrt{V}$ to find that
\begin{align*}
\int_0^{1/2} \frac{e^{\eta R - \eta^2 V} -1}{\eta} \dif \eta
&~\ge~
\left(e^{\frac{R^2}{4 V}-\frac{1}{2}}-1\right) \ln
\left(\frac{1}{1-\frac{\sqrt{2 V}}{R}}\right)
+\ln \left(\frac{R}{V}\right)
\\
&~\ge~
\left(e^{\frac{R^2}{4 V}-\frac{1}{2}}-1\right) \ln
\left(\frac{1}{1-\frac{\sqrt{2 V}}{R}}\right)
-\thalf \ln \left(\frac{V}{4}\right)
.
\end{align*}
Using 
$e^{\half (x^2-1)} = e^{\half (x-1)^2} e^{x-1} \ge e^{\half(x-1)^2} x$,
$-1 \ge -\frac{R}{\sqrt{2 V}}$ and $\ln \frac{1}{1-x} \ge x$, we find
\[
\left(e^{\frac{R^2}{4 V}-\frac{1}{2}}-1\right) \ln
\left(\frac{1}{1-\frac{\sqrt{2 V}}{R}}\right)
~\ge~
\left(e^{\frac{1}{2} \del*{\frac{R}{\sqrt{2 V}}-1}^2} \frac{R}{\sqrt{2 V}} - \frac{R}{\sqrt{2 V}}\right) \frac{\sqrt{2 V}}{R}
~=~
e^{\frac{1}{2} \del*{\frac{R}{\sqrt{2 V}}-1}^2} - 1
.
\]
Putting everything together and using $V \leq T$ together with $1+\thalf
\ln \frac{T}{4} \leq \thalf + \ln(T+1)$ for $T \geq 1$,
we get
\begin{align*}
0 
&\geq
\pi(\refexperts) \del*{
  e^{\frac{1}{2} \del*{\frac{R}{\sqrt{2 V}}-1}^2} - 1
  - \thalf \ln \frac{V}{4}
}
-(1-\pi(\refexperts)) \del*{\thalf + \ln (T+1)}\\
&\geq
\pi(\refexperts) \del*{
  e^{\frac{1}{2} \del*{\frac{R}{\sqrt{2 V}}-1}^2}}
-\del*{\thalf + \ln (T+1)},
\end{align*}
which implies
\[
R
~\le~
\sqrt{2 V} \del*{1 + \sqrt{2\ln \del*{\frac{\half + \ln (T+1)}{\pi(\refexperts)}}}}
,
\]
and \eqref{eqn:improper.bound} follows.

It remains to consider the case that $\mleta > 1/2$, which implies $R >
V$. We then use $v=1/2$, for which \eqref{eqn:NoNastyEta} leads to
\[
\int_0^{1/2} \frac{e^{\eta R - \eta^2 V} -1}{\eta} \dif \eta
~\ge~
(e^{u R - u^2 V} - 1) \ln \frac{1}{2 u}   %
~\ge~
(e^{u (1-u) R} - 1) \ln\frac{1}{2 u}
.
\]
Putting everything together then gives
\[
R
~\le~
\frac{1}{u (1 - u)}
\ln \del*{1 + \frac{(1-\pi(\refexperts))\big(\thalf +  \ln (T+1)\big)}
  {- \ln (2 u) \pi(\refexperts)}}
~\le~
\frac{1}{u (1 - u)}
\ln \del*{1 + \frac{\thalf +  \ln (T+1)}
  {- \ln (2 u) \pi(\refexperts)}}
.
\]
And \eqref{eqn:improper.bound} follows by plugging in
$u=\frac{5-\sqrt{5}}{10}$.
\end{proof}

\subsection{Lemma~\ref{lem:CB.bound}}

\begin{proof}
Note that \eqref{eq:CB.unconstrained} is minimized at the independent component-wise posteriors
\begin{equation}\label{eq:unc.min}
\tilde u_{t+1}^k
~=~
\frac{
  u_t^k e^{-x_t^{k,1}}
}{
  u_t^k e^{-x_t^{k,1}} +
  (1- u_t^k) e^{-x_t^{k,0}}
}
.
\end{equation}
The instantaneous mix loss regret in coordinate $k$ in round $t$ hence equals
\begin{align*}
&
- \ln \del*{  u_t^k e^{-x_t^{k,1}} +
  (1- u_t^k) e^{-x_t^{k,0}}}
- v^k x_t^{k,1}
- (1-v^k) x_t^{k,0}
\\
&~=~
v^k \ln \frac{\tilde u_{t+1}^k}{u_t^k}
+ (1-v^k) \ln \frac{1-\tilde u_{t+1}^k}{1-u_t^k}
\\
&~=~
\re_2 \delcc{v^k}{u_t^k} - \re_2 \delcc{v^k}{\tilde u_{t+1}^k}
,
\end{align*}
and we can write the cumulative regret as
\begin{align}\label{eq.CB.regret}
\sum_{t=1}^T \sum_{k=1}^K \del[\Big]{
  \re_2 \delcc{v^k}{u_t^k} - \re_2 \delcc{v^k}{\tilde u_{t+1}^k}
}
&~=~
\sum_{t=1}^T \del[\Big]{
  \re_2 \delcc{\v}{\u_t} - \re_2 \delcc{\v}{\tilde \u_{t+1}}
}.
\end{align}
As $\re_2$ is a Bregman divergence (for convex generator $F(\x) ~=~ \sum_k x_k \ln x_k + (1-x_k) \ln (1-x_k)$), it is non-negative and satisfies the
generalized Pythagorean inequality for Bregman divergences \citep[Lemma~11.3]{CesaBianchiLugosi2006}. Since $\v \in \U$ and $\u_{t+1}$
is the projection of $\tilde \u_{t+1}$ onto $\U$, these properties
together imply that
\[
\re_2 \delcc{\v}{\u_{t+1}}
~\le~
\re_2 \delcc{\v}{\u_{t+1}}
+
\re_2 \delcc{\u_{t+1}}{\tilde \u_{t+1}}
~\le~
\re_2 \delcc{\v}{\tilde \u_{t+1}}
.
\]
Hence the cumulative mix loss regret satisfies
\begin{align*}
\eqref{eq.CB.regret}
&~\le~
\sum_{t=1}^T \del[\Big]{
  \re_2 \delcc{\v}{\u_t} - \re_2 \delcc{\v}{\u_{t+1}}
}
~=~
\re_2 \delcc{\v}{\u_1} - \re_2 \delcc{\v}{\u_{T+1}}
~\le~
\re_2 \delcc{\v}{\vpi},
\end{align*}
as required.
\end{proof}

\subsection{Lemma~\ref{lem:CSquint.negpot}}
\begin{proof}
First, observe that, for any $\eta$,
\begin{align*}
e^{- \frac{1}{K} \lcmix(\u_t^\eta, \x_t^\eta)}
&~\stackrel{\text{\tiny \eqref{eq:mix.loss.def}, Jensen}}{\le}~
  \frac{1}{K} \sum_{k=1}^K \del*{u_t^{\eta,k} e^{-x_t^{\eta,k,1}} + (1-u_t^{\eta,k}) e^{- x_t^{\eta, k,0}}}
\\
&~\stackrel{\text{\tiny \eqref{eq:loss.taste}}}{=}~
\frac{1}{K} \sum_{k=1}^K \del*{u_t^{\eta,k} (1+\eta r^{k,1}) + (1-u_t^{\eta,k}) (1+ \eta r^{k,0})}
\\
&~\stackrel{\text{\tiny \eqref{eq:def.regret}}}{=}~
1 +
\frac{\eta}{K} \sum_{k=1}^K \del*{u_t^k  - u_t^{\eta,k}} \loss^k
~=~
1 +
\frac{\eta}{K}  \del*{\u_t  - \u_t^{\eta}}^\top \vloss.
\end{align*}
We hence have
\begin{align*}
\Phi(\r_{1:T+1}) - \Phi(\past)
&~=~
\ex_{\gamma(\eta)} \sbr*{
  e^{- \frac{1}{K} \sum_{t=1}^T \lcmix(\u_t^\eta, \x^\eta)}
  \del*{
  e^{- \frac{1}{K} \lcmix(\u_{T+1}^\eta, \x_{T+1}^\eta)}
  - 1
  }
}
\\
&~\le~
\ex_{\gamma(\eta)} \sbr*{
  e^{- \frac{1}{K} \sum_{t=1}^T \lcmix(\u_t^\eta, \x^\eta)}
  \frac{\eta}{K} \del*{\u_{T+1}  - \u^{\eta}_{T+1}}^\top \vloss
}
~=~ 0,
\end{align*}
where the last equality is by design of the weights \eqref{eq:CSquint}.
\end{proof}

\subsection{Lemma~\ref{lem:CSquint.regret.bound}}
\begin{proof}
Lemma~\ref{lem:CSquint.negpot} tells us that Component iProd ensures
$\Phi(\past) \le 0$. For any $\eta,$ this implies
\begin{align*}
- K \ln \gamma(\eta)
&~\stackrel{\text{\tiny Lemma~\ref{lem:CSquint.negpot}}}{\ge}~
- \sum_{t=1}^T \lcmix(\u_t^\eta, \x_t^\eta)
~\stackrel{\text{\tiny Lemma~\ref{lem:CB.bound}}}{\ge}~
- \sum_{t=1}^T \sum_{k=1}^K \del*{v^k x_t^{\eta,k,1} + (1-v^k) x_t^{\eta,k,0}}
-  \re_2\delcc{\v}{\vpi}
\\
&~\stackrel{\text{\tiny \eqref{eq:loss.taste},\eqref{eq:prodbound}}}{\ge}~
-  \sum_{t=1}^T \sum_{k=1}^K \del*{
  v^k \del*{\eta^2 (r_t^{k,1})^2 - \eta r_t^{k,1}}
  + (1-v^k) \del*{\eta^2 (r_t^{k,0})^2 - \eta r_t^{k,0}}
}
- \re_2\delcc{\v}{\vpi}
\\
&~=~
  \eta R_T^\v
- \eta^2 V_T^\v
- \re_2\delcc{\v}{\vpi},
\end{align*}
from which the result follows.
\end{proof}

\subsection{Theorem~\ref{thm:CSquint.final.regret.bound}}
\begin{proof}
The exponentially spaced grid of learning rates ensures that, for any
$\eta \in [\tfrac{1}{2T},\thalf]$, there always exists an $\alpha \in
[\thalf,1]$ for which $\alpha \eta$ is a grid point. Hence, whenever
$\hat \eta = \frac{R_T^\v}{2 V_T^\v} \in [\tfrac{1}{2T},1/2]$,
\eqref{eqn:CSquint.discretised.regret.bound} implies that
\[
R_T^\v
~\le~
\frac{4}{\sqrt{3}}
\sqrt{
  V_T^\v
  \del[\big]{
    \re_2\delcc{\v}{\vpi}
    + K \ln \ceil{1 + \log_2 T}
  }
}
,
\]
and \eqref{eqn:CSquint.final.regret.bound} is satisfied.
Alternatively, if $\hat \eta < \tfrac{1}{2T}$, then $R_T^\v < V_T^\v/T \le K$, 
and \eqref{eqn:CSquint.final.regret.bound} again holds. Finally,
supppose that $\hat \eta > 1/2$. Then $R_T^\v > V_T^\v$, and plugging
$\eta=1/2$ into Lemma~\ref{lem:CSquint.regret.bound} results in
\[
\frac{1}{2} R_T^\v
- \frac{1}{4} V_T^\v
~\le~
\re_2\delcc{\v}{\vpi}
+ K \ln \ceil{1 + \log_2 T}
.
\]
Using that $R_T^\v > V_T^\v$, the left-hand side is at most
$\frac{1}{4}R_T^\v$, from which \eqref{eqn:CSquint.final.regret.bound}
follows.
\end{proof}

\section{Numerical stability}\label{eq:numerics}
Although we are not numerical specialists, it is clear that some care should be taken evaluating the weight expressions for the conjugate prior \eqref{eq:conjugate.weights} and improper prior \eqref{eq:Squint.improper}. 
Initially, and as long as $V=0$, we have $R=0$ and hence by \eqref{eq:Squint} Squint sets the weights $\w$ equal to the prior $\pi$. We now assume $V>0$, and look at \eqref{eq:conjugate.weights} and \eqref{eq:Squint.improper}. Both involve a contribution of the form
\begin{equation}\label{eq:xi}
\frac{\sqrt{\pi } e^{\frac{R^2}{4 V}} \left(\erf\left(\frac{R}{2 \sqrt{V}}\right)-\erf\left(\frac{R-V}{2 \sqrt{V}}\right)\right)}{2 \sqrt{V}}
.
\end{equation}
This expression is empirically numerically stable unless both $\erf$ arguments fall outside $[-6,6]$ to the same side. In other words, it can be used when
\begin{equation}\label{eq:stability}
-6
~\le~ \frac{R}{2 \sqrt{V}}
\quad
\text{and}
\quad
\frac{R-V}{2 \sqrt{V}}
~\le~
6
,
\qquad
\text{that is}
\qquad
R \in [-12 \sqrt{V}, V + 12 \sqrt{V}]
.
\end{equation}
If we are not in this range, then we are feeding extreme arguments into both $\erf$s. Hence we may Taylor expand \eqref{eq:xi} around $R = \pm \infty$ (both of which give the same result) to get
\[
\frac{e^{\frac{R}{2}-\frac{V}{4}}-1}{R}
\quad
\text{($0$th and $1$st order)}
\qquad
\text{or}
\qquad
\frac{e^{\frac{R}{2}-\frac{V}{4}} (R+V)-R}{R^2}
\quad
\text{($2$nd order)}
.
\]
Note that this $0$th order expansion is negative for $R \in [0, V/2]$, but that falls well within the stable range \eqref{eq:stability} where we should use \eqref{eq:xi} directly.

\end{document}